\documentclass[conference]{IEEEtran}
\pdfoutput=1
\usepackage[utf8]{inputenc}

\usepackage{comment}
\usepackage{xspace}
\usepackage{graphicx}
\usepackage{overpic}
\usepackage{svg}
\usepackage{xcolor}
\usepackage{subcaption}
\definecolor{gg}{HTML}{2b3b5e}
\definecolor{Gray}{gray}{0.95}
\usepackage{soul}
\usepackage{float}

\definecolor{setgreen}{HTML}{66c2a5}
\definecolor{setorange}{HTML}{fc8d62}
\definecolor{setblue}{HTML}{8da0cb}
\definecolor{setpurp}{HTML}{e78ac3}
\definecolor{AddedText}{HTML}{00792F}

\usepackage{tabularx}
\usepackage{multirow}
\usepackage{booktabs}
\usepackage{colortbl}
\usepackage{siunitx}

\usepackage{etoolbox}
\makeatletter
\patchcmd{\@makecaption}
{\scshape}
{}
{}
{}
\makeatother

\let\labelindent\relax
\usepackage[inline]{enumitem}

\usepackage{amsmath}
\usepackage{amssymb}
\usepackage{amsthm}

\newtheoremstyle{main}
{1em}                                                %
{1em}                                                %
{\itshape}                                        %
{0pt}                                                %
{\scshape}                                           %
{\\*}                                                %
{2pt}                                                %
{\thmname{#1}\thmnumber{ #2}: \thmnote{\itshape #3}} %

\newtheorem{lemma}{Lemma}[section]

\usepackage[linesnumbered,ruled,noend]{algorithm2e}

\usepackage[
	activate   = {true},
	protrusion = false,
	expansion  = true,
	kerning    = true,
	spacing    = true,
	tracking   = false,
	auto       = true,
	selected   = true,
	factor     = 1000,
	stretch    = 20,
	shrink     = 20,
]{microtype}

\usepackage[english=american]{csquotes}
\MakeOuterQuote{"}
\usepackage[
	maxbibnames=99,
	maxcitenames=2,
	natbib=true,
	style=numeric-comp,
	backend=biber,
	sorting=none,
	giveninits=true,
	url=false,
	doi=false,
	eprint=true,
	isbn=false,
]{biblatex}

\makeatletter
\let\NAT@parse\undefined
\makeatother
\usepackage[pdfa,colorlinks,bookmarksopen,bookmarksnumbered,allcolors=gg]{hyperref}
\usepackage[english]{babel}

\usepackage[nameinlink,capitalise]{cleveref}
\crefname{line}{line}{lines}
\crefname{figure}{Fig.}{Figs.}
\Crefname{figure}{Fig.}{Figs.}
\crefname{equation}{Eq.}{Eqs.}
\Crefname{equation}{Eq.}{Eqs.}
\crefname{section}{Sec.}{Secs.}
\Crefname{section}{Sec.}{Secs.}
\crefname{definition}{Def.}{Defs.}
\Crefname{definition}{Def.}{Defs.}
\crefname{algorithm}{Alg.}{Algs.}
\Crefname{algorithm}{Alg.}{Algs.}
\crefname{assumption}{Asm.}{Asms.}
\Crefname{assumption}{Asm.}{Asms.}
\crefname{subassumption}{Asm.}{Asms.}
\Crefname{subassumption}{Asm.}{Asms.}

\newcommand{\mf}[1]{\mbox{\cref{#1}}\xspace}

\usepackage{flushend}

\newcommand\blfootnote[1]{%
  \begingroup
  \renewcommand\thefootnote{}\footnote{#1}%
  \addtocounter{footnote}{-1}%
  \endgroup
}

\graphicspath{ {./figures} }

\newcommand{\simd}{\textsc{simd}\xspace}
\newcommand{\simt}{\textsc{simt}\xspace}
\newcommand{\cpu}{\textsc{cpu}\xspace}
\newcommand{\gpu}{\textsc{gpu}\xspace}
\newcommand{\amd}{\textsc{amd}\xspace}

\newcommand{\sbmp}{\textsc{sbmp}\xspace}
\newcommand{\sbmps}{\textsc{sbmp}s\xspace}

\newcommand{\fcl}{\textsc{fcl}\xspace}

\newcommand{\eg}{\emph{e.g.},\xspace}
\newcommand{\ie}{\emph{i.e.},\xspace}
\newcommand{\dof}{\textsc{d\scalebox{.8}{o}f}\xspace}

\newcommand{\flann}{\textsc{flann}\xspace}

\newcommand{\nerfs}{\textsc{n\scalebox{0.8}{e}rf}s\xspace}
\newcommand{\nanoflann}{Nano\textsc{flann}\xspace}
\newcommand{\gnat}{\textsc{gnat}\xspace}

\newcommand{\delete}[1]{}
\newcommand{\add}[1]{#1}

\newcommand{\nameoftechnique}{collision-affording point tree}

\newcommand{\capt}{\textsc{capt}\xspace}
\newcommand{\capts}{\textsc{capt}s\xspace}

\usepackage{mathtools}

\usepackage{siunitx}
\begin{document}

\title{Collision-Affording Point Trees: SIMD-Amenable Nearest Neighbors for Fast Collision Checking}

\author{
  \authorblockN{Clayton W. Ramsey, Zachary Kingston\authorrefmark{1}, Wil Thomason\authorrefmark{1}, and Lydia E. Kavraki}
  \authorblockA{Rice University\\
  \texttt{\{clayton.w.ramsey,zak,wbthomason,kavraki\}@rice.edu}\\
  \authorrefmark{1}Equal Contribution.}
}

\maketitle

\begin{abstract}
Motion planning against sensor data is often a critical bottleneck in real-time robot control.
For sampling-based motion planners, which are effective for high-dimensional systems such as manipulators, the most time-intensive component is collision checking.
We present a novel spatial data structure, the \nameoftechnique{} (\capt{}): an exact representation of point clouds that accelerates collision-checking queries between robots and point clouds by an order of magnitude, with an average query time of less than 10 nanoseconds \add{on 3D scenes comprising thousands of points}.
With the \capt, sampling-based planners can generate valid, high-quality paths in under a millisecond, with total end-to-end computation time faster than 60 FPS, on a single thread of a consumer-grade \cpu.
We also present a point cloud filtering algorithm, based on space-filling curves, which reduces the number of points in a point cloud while preserving structure.
Our approach enables robots to plan at real-time speeds in sensed environments, opening up potential uses of planning for high-dimensional systems in dynamic, changing, and unmodeled environments.
\end{abstract}

\blfootnote{
    This work was supported by NSF CCF 2336612, NSF ITR 2127309 for the CRA CIFellows Project, and Rice University Funds.
}

\section{Introduction}\label{sec:intro}
Motion planning underpins many applications of high-degree-of-freedom robots, allowing them to efficiently find collision-free trajectories between arbitrary poses.
Modern motion planning methods capably solve problems with many obstacles for these high-dimensional robots, typically by either building and searching a graph or tree approximating the collision-free subset of the robot's state space (\ie{} sampling-based motion planning (\sbmp)~\cite{orthey2023sampling,LaValle2001,Kavraki1996}) or by solving a numerical optimization problem (\ie{} trajectory optimization~\cite{Schulman2014,Zucker2013,bhardwaj_storm_integrated_2021}).
The most time-consuming component of most motion planners---and \sbmps{} in particular---is \emph{state validation}, which ensures that a robot's state does not violate its constraints or collide with obstacles~\cite{bialkowski_massively_parallelizing_2011,lavalle2006planning}.
State validators commonly assume knowledge of the precise geometries and positions of all obstacles in the environment---an assumption that does not hold for general real-world settings where only sensed representations of the world may be available.
Earlier work that checks for collisions between robot geometries and point clouds~\cite{Schauer2015,Hornung2013,Pan2012a,Pan2013,Danielczuk2021,Murali2023} attempts to lift this assumption, but point cloud collision checking remains a relatively slow bottleneck for motion planning.

Recent work~\cite{Thomason2023VAMP,sundaralingam2023curobo,Vasilopoulos2023,le2024accelerating} has produced new approaches to hardware-accelerated motion planning that exploit parallelism in collision checking and other core planning operations to find complete trajectories in microseconds to milliseconds.
This use of parallelism motivates the need for higher-throughput parallelism-friendly algorithms and data structures for efficiently planning collision-free motions in environments that are only perceived as point clouds, \eg{} from a common depth camera like the Intel RealSense.
\citet{sundaralingam2023curobo} use \gpu{} parallelism to allow batched querying of an approximate Euclidean Signed Distance Field (\textsc{esdf}) built from a series of sensor measurements~\cite{Millane2023}; similarly, \citet{Vasilopoulos2023} perform a \gpu{}-parallel brute-force \textsc{sdf} computation between a discretized set of points on the robot geometry and a dense point cloud.
Although these \gpu{}-based methods are promising for some applications, data synchronization costs between the \gpu{} and \cpu{} limit the direct applicability of these techniques for motion planning algorithms, many of which are \cpu{}-based.
Further, for applications of field robotics such as planetary rovers, agricultural robots, and others, the power requirements of an onboard \gpu{} may be untenable.

In this work, we propose a data structure and associated construction and search algorithms for \emph{exact} point cloud distance computation and collision checking.
Our proposed data structure, the \emph{\nameoftechnique} (\capt{}), adapts and refines concepts from the classical $k$-d tree to support efficient parallel evaluation.
The core insights guiding our design of the \capt{} are that
\begin{enumerate*}[label=(\arabic*)]
    \item exploiting the spatial correlation present in \sbmp{} edge validation collision queries allows for aggressive early-termination of batched queries without sacrificing correctness, and that
    \item many motion planning problems (\eg{} in manipulation) only need to represent a relatively small, \emph{local} part of the environment for collision checking
\end{enumerate*}.
This first insight extends ideas from~\citet{Thomason2023VAMP}; the second enables us to rethink traditional assumptions about minimizing memory overhead in point cloud representations: we in fact duplicate select subsets of points to create a parallelism-friendly data model.
By combining these insights with a machine-sympathetic data structure and the use of data-level parallelism, \capts{} can return collision results \add{against an observed 3D point cloud} in a mean time of \textbf{under ten nanoseconds per query} on a single core of a consumer desktop \cpu{}\footnote{\add{The point clouds used to generate these collision-checking throughput results contained up to 50000 points and had mean dispersions between 7\si{\milli \meter} and 2.2\si{\centi \meter}. For detailed analysis, refer to the results in \cref{subsec:throughputresults} and \cref{app:dispersion};  for examples of such point clouds, see \cref{subfig:realsense_filtered} and \cref{subfig:pc_filtered}.}}.
Although we focus on \cpu{}-based \emph{single-instruction, multiple-data} (\simd) parallelism in this paper, our approach also applies to and may benefit \gpu{}-based planners using a \emph{single-instruction, multiple-thread} (\simt) parallelism model.

Concretely, we contribute:
\begin{enumerate}
	\item the \nameoftechnique{} (\capt{}), a novel data structure for storing sensed point clouds for collision checking.
	\item  efficient construction, branch-free parallel query, and collision check algorithms for \capts.
	\item a method for efficient point cloud down-sampling by exploiting properties of space-filling curves.
	\item proofs of correctness (\ie{} that using \capts{} and our filtering algorithm does not modify planning problem feasibility).
	\item empirical evaluations for collision throughput and error against a set of competitive baselines.
	\item integration with a vectorized motion planner~\cite{Thomason2023VAMP}, demonstrating the utility of \capts{} for high-performance \sbmp on a number of difficult and cluttered problems.
	\item \add{a proof-of-concept demonstration with a depth camera and physical robot hardware.}
    \item \add{an open source implementation of \capts{}\footnote{Available at \href{https://github.com/KavrakiLab/vamp}{https://github.com/kavrakilab/vamp}.}.}
\end{enumerate}
\begin{figure*}
    \centering
    \subfloat[]{
        \includegraphics[width=0.26\textwidth]{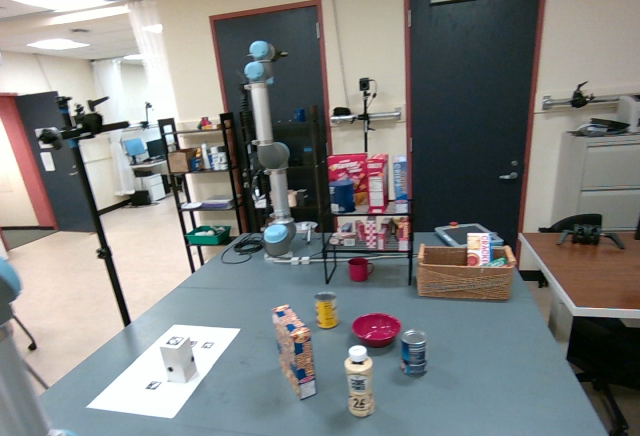}
        \label{subfig:realsense_color}
    }
    \hspace{2mm}
    \subfloat[]{
        \includegraphics[width=0.28\textwidth]{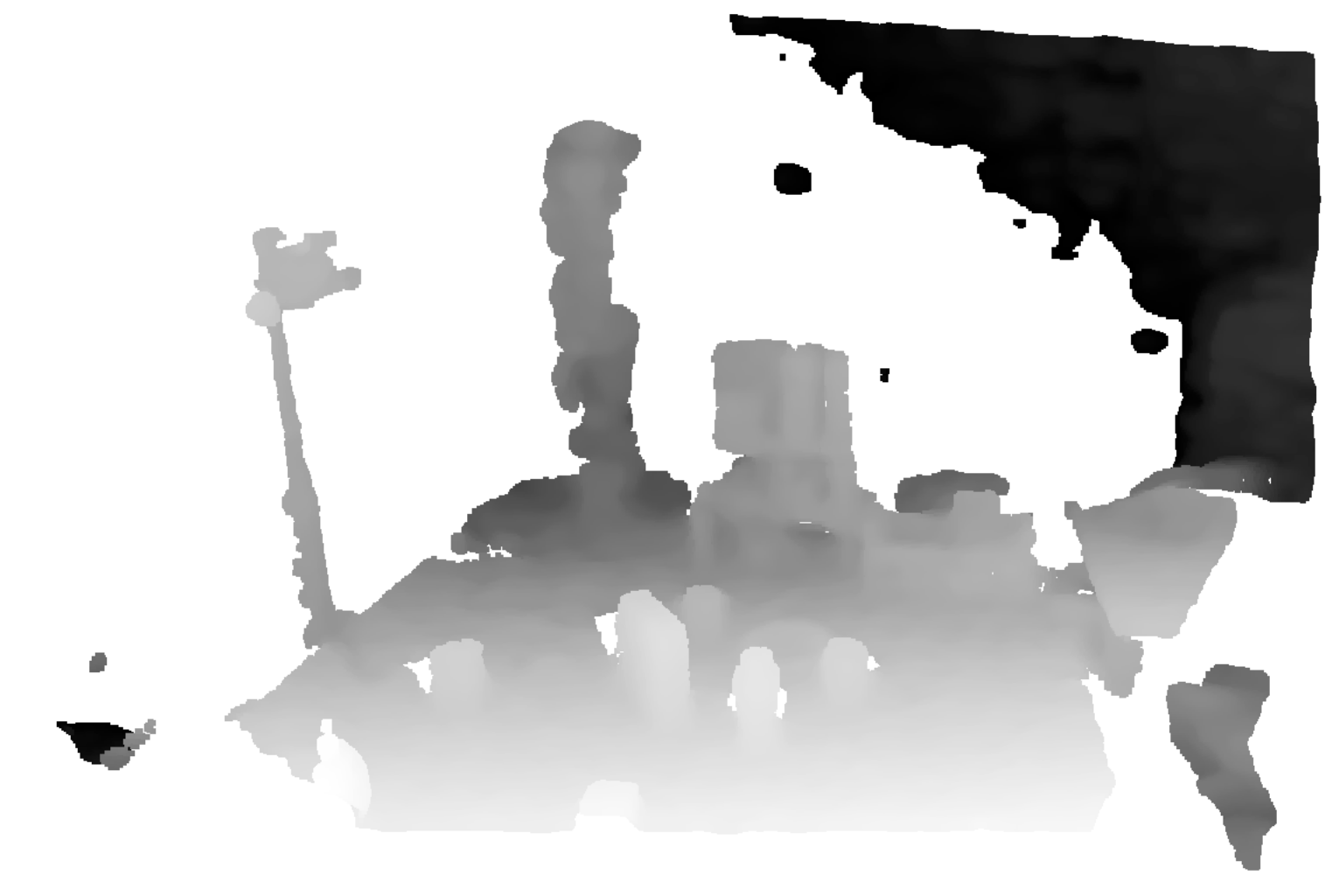}
        \label{subfig:realsense_depth}
    }
        \hspace{2mm}
    \subfloat[]{
        \includegraphics[width=0.28\textwidth]{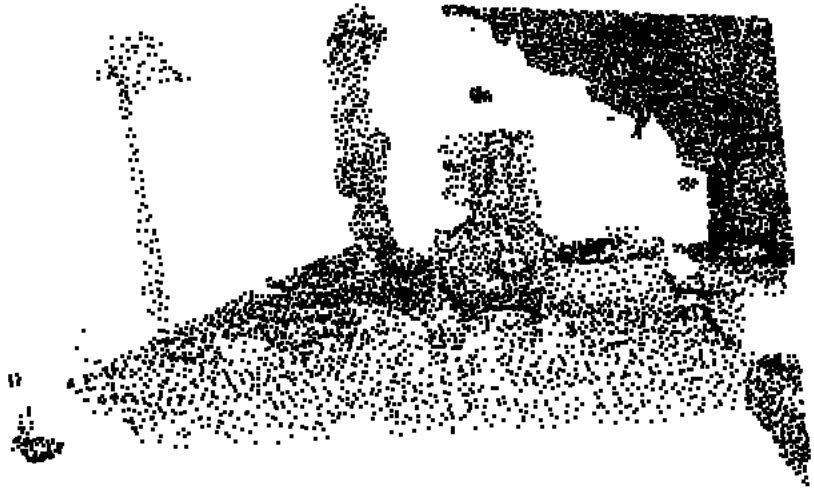}
        \label{subfig:realsense_filtered}
    }
    \caption{
        \ref{subfig:realsense_color}, \ref{subfig:realsense_depth}: A cluttered tabletop scene, captured as RGB-D with an Intel Realsense D455 sensor. 
        \ref{subfig:realsense_filtered}: The point cloud rendering of the same scene, filtered using our proposed space-filling curve method (\cref{sec:filtering}). 
    }
    \label{fig:realsense}
\end{figure*}

\section{Related Work}\label{sec:related}

Despite the inherent computational complexity of motion planning, which is known to be \textsc{pspace}-complete~\cite{reif1979complexity,canny1988complexity}, \sbmps{}~\cite{Kavraki1996,LaValle2001} are nonetheless capable of solving many motion planning problems in tens to hundreds of milliseconds, given explicit representations of the problem environment.
\citet{Thomason2023VAMP} proposed an approach to accelerating \sbmps{} via \simd{} parallelism, decreasing planning times to the order of microseconds---sufficient for real time operation.
Much of this speedup is due to the use of fine-grained parallelism in collision checking; however, this planner still requires an explicit, known model of the environment's geometry.
Particularly, in real-world planning problems, explicit geometry models are often unavailable---thus, planning from sensor data has been desirable since the beginning of the field~\cite{Thrun2002}.

\subsection{Planning from sensor data}

Planning-amenable representations of sensor data usually take the form of some space-partitioning data structure.
$k$-d trees~\cite{Bentley1975} are particularly relevant data structures for point cloud representation.
They are often used for nearest-neighbor search in low-dimensional spaces~\cite{Ram2019, Chen2019, Pinkham2020}, and allow for logarithmic-time collision-checking against points in a point cloud~\cite{Schauer2015}.
These trees can also be augmented with a sphere covering~\cite{Klein2004} for efficient rejection of non-colliding geometry.
Many efforts have been made to accelerate $k$-d tree queries by low-level optimization, including through parallel subtree search~\cite{Chen2023a} and dedicated instruction-set architectures~\cite{E.Becker2023}.
\flann{}~\cite{muja2009flann,blanco2014nanoflann} and Nigh~\cite{ichnowski2018nigh} are popular implementations of $k$-d trees; the former supports approximate nearest-neighbor search, while the latter is exact.

The most popular approaches to representing sensor data for \sbmp{} collision checking often use \emph{occupancy maps}~\cite{moravec_occupancy_1985,thrun1996integrating,Rusu2009}.
OctoMaps~\cite{Hornung2013} implement a probabilistic occupancy map using octrees~\cite{meagher1982geometric} of voxels, where each voxel is considered "occupied" based on Bayesian updates computed upon each new point cloud inserted into the map.
They allow for collision-checking based on bounding-volume hierarchies, which use a branch-and-prune search to limit the set of possible points for collision checking.
OctoMaps are also integrated with collision checking libraries, \eg the Flexible Collision Library \cite{Pan2012a}, and popular planning frameworks such as MoveIt~\cite{chitta2012moveit}.
They are a commonly used representation in practice, \eg for underwater vehicles~\cite{VidalGarcia2019}, subterranean exploration~\cite{dang2020graph}, autonomous vehicles~\cite{badue2021self}, and aerial vehicles~\cite{lu2018survey}. 

Voxel-based approaches have also seen significant use recently due to 
many optimization-based planners~\cite{Zucker2013, Vasilopoulos2023} using signed distance fields for collision-checking---the representation of the signed-distance field is backed by a voxel representation~\cite{Newcombe2011, Whelan2015}.
CuRobo \cite{sundaralingam2023curobo} is an optimization-based planner which uses \gpu{} parallelism for high-efficiency collision-checking, including against point clouds.
It relies on the NVBlox~\cite{Millane2023} \gpu-accelerated signed-distance field library for its collision-checking and optimization.
VoxBlox~\cite{Oleynikova2017} and VoxGraph~\cite{Reijgwart2020} are also \cpu{}-based sensor-based mapping tools for constructing signed-distance fields for collision-checking.

Recently, there has also been interest in using implicit representations of the environment, such as neural radiance fields (\nerfs{})~\cite{Mildenhall2021}, which have been used for motion planning~\cite{Adamkiewicz2022, Chen2023}.
However, constructing \nerfs{}, while relatively fast~\cite{Mueller2022}, still take on the order of seconds to construct with \gpu hardware, making them infeasible for online planning.

\subsection{Space-filling curves}

Space-filling curves are continuous real bijections that map every point of a one-dimensional line to a higher-dimensional space, such as $\mathbb{R}^3$. 
Z-order curves~\cite{morton1966computer}, also known as Morton curves, are a class of space-filling curve often used for nearest-neighbor applications.
A point in a high-dimensional space can be projected onto a Z-order curve by interleaving the bits of the binary representation of its coordinates.
If two points' projections onto the curve are close together, they are likely to also be near in the higher-dimensional pre-image space.
\citet{Ying2023} used a Z-order curve to produce a low-discrepancy sub-sampling of a point cloud; however, their sub-sampling procedure does not guarantee that the overall point cloud structure is preserved.
Likewise,~\citet{connor2010morton} used a Z-order sorting to accelerate construction of $k$-nearest-neighbor graphs by limiting their search to a single range of the space-filling curve.

\subsection{SIMD parallelism and planning}

Parallel motion planning algorithms~\cite{Thomason2023VAMP,sundaralingam2023curobo,Ichnowski2012,Plaku2005,amato_probabilistic_roadmap_1999,jacobs_scalable_method_2012,bialkowski_massively_parallelizing_2011,Pan2012,le2024accelerating} require parallelizable collision-checking data structures.
Further, to reap the benefits of early-termination in collision checking that enable \simd-accelerated planning to plan at the microsecond scale, a \simd-amenable data structure is required.
However, the major data structures used in planning for sensor data representation are not amenable to this flavor of parallelism---hierarchical space-subdividing data structures (\eg~OctoMaps, $k$-d trees, other voxel grids) require conditional-branch-heavy searches through large subtrees, which result in a highly sub-optimal memory access pattern: collision checks in these data structures must access potentially many fragmented segments of memory.
Additionally, the branching nature of these searches makes them impractical to adapt to the branchless computation framework that \simd{} parallelism best matches.
To overcome these issues, we present a new data structure, designed with branchless, cache-friendly access in mind and demonstrate its effectiveness for efficient collision checking.

\section{Preliminaries}\label{sec:prelim}

A $k$-d tree~\add{\cite{friedman1977kdtree}} is a class of space partitioning tree which represents and organizes a set of $n$ points, $\mathit{PC} \subseteq \mathbb{R}^k$.
Each leaf of the tree corresponds to a \emph{cell}, which is an axis aligned bounding box in $\mathbb{R}^k$ containing exactly one \emph{representative point}, $p_r \in \mathit{PC}$.
The union of a $k$-d tree's cells spans the entirety of $\mathbb{R}^k$
\footnote{Slightly abusing the notion of an AABB, the cells on the boundary of the $k$-d tree are half-open.} and the union of all representative points is equal to $\mathit{PC}$.
Each branch of the tree partitions $\mathbb{R}^k$ about an axis-aligned hyperplane: a branch at depth $d$ in the tree with test value $t$ partitions the space such that, for any point $p \in \mathit{PC}$, if $p\lbrack d \mod k\rbrack \leq t$ (where $p\lbrack i \rbrack$ is the $i$-dimension component of $p$), it belongs to the left sub-tree of the branch; otherwise, it belongs to the right sub-tree.

We can construct a $k$-d tree from a point cloud $\mathit{PC}$ in $O(k n \log n)$ time using a recursive partitioning algorithm.
At each level, we split $\mathit{PC}$ into two equally-sized subsets $B_1, B_2$ by a hyperplane intersecting the median point in $\mathit{PC}$ along axis $d$. 
This splitting continues recursively for $B_1$ and $B_2$ until the point cloud contains exactly one point. 
Each splitting hyperplane forms one branch in the tree, and the walls of each leaf's cell are the splitting hyperplanes of its parent branches for each dimension.

$k$-d trees are most commonly used for nearest-neighbor search. 
Given some query point $x$, a recursive branch-and-bound search through the tree can find the closest element of $\mathit{PC}$ to $x$ in $O(\log n)$ operations~\add{\cite{friedman1977kdtree}}.

We can use the nearest-neighbor facility of a $k$-d tree to check for collisions between a sphere and a point cloud.
Given a sphere centered at $x$ with radius $r$, the aforementioned branch-and-bound search procedure can find the nearest point $p \in \mathit{PC}$ to $x$. 
Then, if $\|x - p\| \leq r$, the sphere is in collision; otherwise,  there must be no point in $\mathit{PC}$ whose distance to $x$ is less than $r$, and $x$ therefore does not collide with $\mathit{PC}$.

\section{Method}\label{sec:method}
\capts{} redesign aspects of the classic $k$-d tree to make it amenable to high-throughput parallel querying for robot collision checking.
The major problems with directly using a $k$-d tree for this purpose are
\begin{enumerate*}[label=(\arabic*)]
    \item cache-unfriendly random memory access patterns that arise from the tree's storage representation, and
    \item the inherently conditional-branch-heavy backtracking recursive algorithm used for normal $k$-d tree nearest-neighbor queries
\end{enumerate*}.
Accordingly, the defining features of a \nameoftechnique, as opposed to a $k$-d tree, are its use of a memory layout that improves cache coherency during tree traversal, and that each leaf of a \nameoftechnique{} contains an \emph{affordance set}, a conservative approximation of the possible nearest-neighbors to any point in the cell.
This deceptively simple change allows the implementation of a search to avoid the backtracking stage of a search through a $k$-d tree, enabling branch-free, parallelism-friendly exact collision checking in a fraction of the time.

To use a \capt{} to check for robot collision, we first assume that the robot is made up of some set of spheres $S$ (as in other motion planning work, \eg~\cite{Thomason2023VAMP,sundaralingam2023curobo,mukadam_continuoustime_gaussian_2018}).
Non-spherical robots can be approximated conservatively by constructing a spherical bounding volume hierarchy~\cite{Bradshaw2004} which contains all of their collision geometry.
Let the smallest sphere of the robot's geometry have some known radius $r_\text{min}$ and the largest sphere of the robot's geometry have some known radius $r_\text{max}$.
Then, given a query sphere with center $x$ and radius $r: r_\text{min} \leq r \leq r_\text{max}$, we can use the \capt{} to check if the sphere is in collision.
First, we search through the tree to find the leaf cell of the tree which contains $x$.
If any point in the leaf's affordance set collides with the sphere, then the sphere is in collision; otherwise, the query sphere is not in collision.

We begin by explaining the \nameoftechnique{} and its construction process in \cref{sec:capt,sec:construction}. 
Next, we detail a point cloud filtering algorithm based upon space-filling curves in \cref{sec:filtering}.
Lastly, we describe our branch-free parallel collision-checking algorithm for \capts{} in \cref{sec:search}.

\subsection{The collision-affording point tree}\label{sec:capt}

\begin{figure}
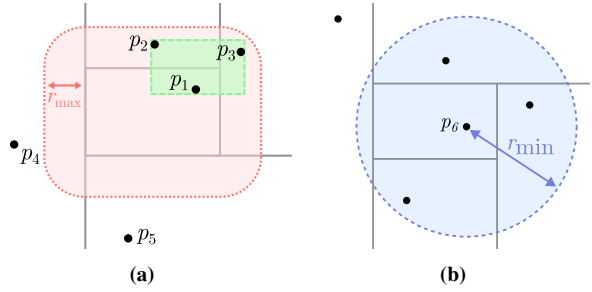

    \centering
    \subfloat[]{
        \includesvg[height=0.18\textwidth]{figures/affordance}
        \label{subfig:affordance}}
    \hspace{2mm}
    \subfloat[]{
    \includesvg[height=0.18\textwidth]{figures/rmin}
        \label{subfig:rmin}}
    \caption{
        \ref{subfig:affordance}: The cell containing $p_1$ affords $p_2$ and $p_3$ at radius $r_\text{max}$, but not $p_4$ or $p_5$. 
        The axis-aligned bounding box containing all afforded points is depicted in green.
        \ref{subfig:rmin}: The sphere centered at $p_6$ of radius $r_\text{min}$ contains the entire cell, so no other points need to be included in the cell's affordance set.
    }\label{fig:affordanceset}
    
\end{figure}

In a \capt{}, each leaf of the tree contains the affordance set for the leaf's corresponding cell, illustrated in~\cref{fig:affordanceset}.
Given the cell $c$ corresponding to a leaf, that leaf's affordance set is the set of all points in $\mathit{PC}$ such that $c$ \emph{affords} collision with $l$ at the radius $r_\text{max}$.
A cell $c$ affords a point $p$ at radius $r$ if there exists some point $q \in c$ such that $\|p - q\| \leq r$, as depicted in \cref{subfig:affordance}.
Intuitively, a point $p$ is afforded by a cell $c$ if a sphere of radius $r$ whose center is contained by $c$ could collide with $p$.
Finally, each leaf is associated with a \emph{second} axis-aligned bounding box. 
This bounding box is \emph{not} the same as the cell; instead, it is the minimal bounding box containing all points in the leaf's affordance set, as shown in green in \cref{subfig:affordance}.

For simplicity, we also assume that $\mathit{PC}$ contains $n$ points, such that $n$ is a power of two. 
If $n$ is not a power of two, we pad $\mathit{PC}$ with points at $\infty$ to the next greatest power of two.

A \capt{} is then the tuple $(T, A, P)$, such that the test sequence $T$ is an array of $n - 1$ test values in $\mathbb{R} \cup \{\infty\}$, $A$ is an array of $n$ axis-aligned bounding boxes over $\mathbb{R}^k$, and the affordance table $P$ is a ragged two-dimensional array of $n$ different affordance sets. 
This representation is \emph{implicit}: the tree does not store any information about its branches other than in the test sequence $T$.

We arrange $T$ according to an Eytzinger layout, a class of array-backed implicit tree layout originally used for heaps~\cite{williams1964heapsort}. 
This layout reduces memory fragmentation by storing all data in a single contiguous block, and also improves performance by allowing for branch-free traversal.
In such a layout, $T_0$ corresponds to the root branch of the tree: all points whose $x$-value is less than $T_0$ belong to the left sub-tree, while all others belong to the right sub-tree.
Next, $T_1$ and $T_2$ correspond to the first branches in the left and right sub-trees about the $y$-value of each point.
Recursively, if $T_i$ corresponds to a partition of the tree about the dimension $d$, then $T_{2i+1}$ corresponds to the next branch in the left sub-tree, while $T_{2i+2}$ corresponds to the next branch in the right sub-tree, both of which split on dimension $d + 1 \mod k$.
\add{The} value of $T_i$ at depth $d$ is the median value of $p\lbrack d \mod k\rbrack$ across all representative points $p$ in its sub-tree.
\add{Intuitively, each $T_i$ partitions the space by a new axis-aligned hyperplane, splitting the points in its subtree in half.
For simplicity, we choose to partition on a repeating sequence of axes every time (first along the $x$-axis, then the $y$-axis, and so on), but could substitute other methods, such as randomly selecting an axis.}

\begin{algorithm}
    \caption{Construct}
    \label{alg:construct}
    \KwIn{Point cloud $\mathit{PC}$ containing $n$ points of dimension $k$, minimum query radius $r_{\text{min}}$, maximum query radius $r_\text{max}$, axis-aligned bounding box $c$, affordance set $z$, uninitialized \capt{} $(T, A, P)$, index $i$, dimension index $d$}
    {

        \If{$|PC| = 1$}{

            \add{
            $x \leftarrow$ only element of $\mathit{PC}$;}

            \If{$\exists{q \in c : \|q - x\| > r_\text{min}}$}{

                \add{$\mathit{PC} \leftarrow \mathit{PC} \cup z$;}
            
            }

            \add{$a \leftarrow $ bounding box containing all points in $\mathit{PC}$;}

            \add{append $\mathit{PC}$ onto $P$;
            }

            append $a$ onto $A$;
            
        }

        \Else{

            \add{$T_i \leftarrow$ median value of $p_d$ for all $p \in \mathit{PC}$;}

            $B_1 \leftarrow \{p \in \mathit{PC} : p_d \leq T_i\}$;

            $B_2 \leftarrow \{p \in \mathit{PC} : p_d > T_i\}$;

            $c_1, c_2 \leftarrow c$;

            shrink the upper bound of $c_1$ on dimension $d$ to $T_i$;
            
            raise the lower bound of $c_2$ on dimension $d$ to $T_i$;

            $z_1 \leftarrow \{p \in z \cup B_2 : \text{$c_1$ affords $p$ at $r_\text{max}$}\}$;
            
            $z_2 \leftarrow \{p \in z \cup B_1 : \text{$c_2$ affords $p$ at $r_\text{max}$}\}$;

            Construct($B_1, r_\text{min}, r_\text{max}, c_1, z_1, (T, A, P), 2i + 1, (d + 1) \mod k$);

            Construct($B_2, r_\text{min}, r_\text{max}, c_2, z_2, (T, A, P), 2i + 2, (d + 1) \mod k$);
        }
        
    }
\end{algorithm}

\subsection{Collision-affording point tree construction}\label{sec:construction}
To construct a \capt{}, we apply the same recursive partitioning approach as in $k$-d tree construction, using quick-select~\cite{hoare1961quickselect} for an expected linear-time selection of the median value for each partition. 
The exact construction algorithm is specified in \cref{alg:construct}.
At each step of the construction procedure, we retain two additional sets of information: the current cell $c$ and the current afforded set $z$. 
$c$ is initialized to the cell containing all of $\mathbb{R}^k$, while $z$ is initialized to the empty set.
Every time we split the point cloud along a median plane, we split $c$ about the same median plane into two adjacent cells, $c_1$ and $c_2$.
Next, we duplicate $z$ to produce two new affordance sets, $z_1$ and $z_2$.
We expand $z_1$ to include all points in the cloud contained by $c_2$, and vice versa for $z_2$.
Finally, we filter out all points from $z_1$ which are not afforded by $c_1$, and likewise with $z_2$.
This process can be thought of as maintaining the set of all points outside of a cell $c$ that are still close enough to $c$ that they may collide with a query sphere centered in $c$.
Once each cell contains exactly one point, we determine the final affordance set of that cell's leaf as the union of the set containing the representative point and all points afforded by the leaf cell. 

We can use knowledge of $r_\text{min}$ to prune the affordance set slightly more than the conservative approximation created by the construction procedure.
If all points $x_c \in c$ are so close to the representative point $p$ that $\|x_c - p\| \leq r_\text{min}$, then all spheres with center $x \in c$ and radius $r \geq r_\text{min}$ will collide with $c$'s representative point $p$, as shown in \cref{subfig:rmin}.
If this is the case, then there is no need to include any points outside of $c$ in the affordance set, since we know that any query sphere centered in $c$ is already in collision.
Therefore, for these sufficiently small cells, we can simply store $\{p\}$ as the affordance set, without including any other points.

\subsection{Filtering}
\label{sec:filtering}

\begin{figure}
    \centering
    \subfloat[]{
        \includegraphics[height=0.18\textwidth]{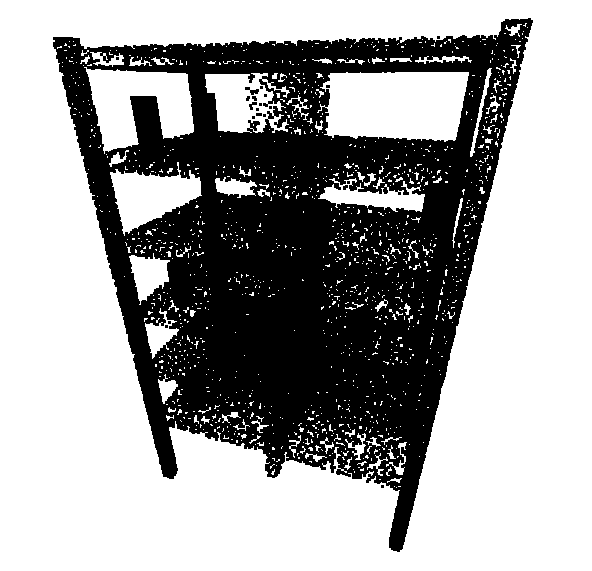}
        \label{subfig:pc_unfiltered}
    }
    \hspace{2mm}
    \subfloat[]{
        \includegraphics[height=0.18\textwidth]{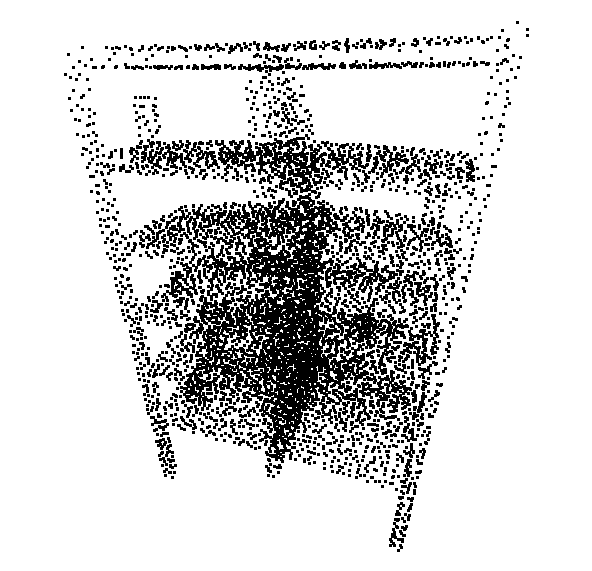}
        \label{subfig:pc_filtered}
    }
    \caption{
        \ref{subfig:pc_unfiltered}: A point cloud with 126000 points, created by randomly sampling the surface of a shelf. 
        \ref{subfig:pc_filtered}: The same point cloud, filtered to 8889 points using $r_\text{filter} = 2 \si{\centi \meter}$. 
    }
    \label{fig:pc}
\end{figure}

Not all points in an input point cloud are required for collision-checking, especially in extremely dense clouds, as there are many redundant points when conservatively approximating the colliding volume of the cloud.
As the construction time of \nameoftechnique{}s grows with the size of the point cloud, we introduce an efficient filtering procedure to significantly reduce the density of the cloud that also guarantees that it will not remove critical points. 

As the cloud is used for collision checking, it cannot remove a point $p$ unless there is another point $p^*$ in the cloud sufficiently close to $p$, such that $\|p - p^*\|$ is less than some threshold radius $r_\text{filter}$.
This guarantees that the robot cannot penetrate further than $r_\text{filter}$ into the point cloud; alternately, all spheres of the robot may be padded by $r_\text{filter}$ to produce a conservative approximation of the cloud.
Lastly, this filtering procedure must be computationally cheap: it must be significantly faster than the construction time of the \capt{} to achieve any useful speedup.

In order to verify that any removed point $p$ has a sufficiently close neighbor $p^*$, we must compute the distance between $p$ and $p^*$. 
A naive algorithm would compare every two points in $\mathit{PC}$, but this approach yields $O(n^2)$ runtime, which is unacceptable for large point clouds.
Instead, we reduce the candidate set of pairs by exclusively comparing points which are relatively near to one another, according to an arbitrary measure of locality.

Space-filling curves (specifically in our implementation, Z-order curves~\cite{morton1966computer}) provide such a measure of locality: we use one to map all points in $\mathit{PC}$ to a one-dimensional space-filling curve. 
Once sorted in order of their position along the curve, points which are adjacent in the curve are also likely to be adjacent in their higher-dimensional space.
Therefore, by exclusively checking the distance between neighboring points in the space-filling curve, we can dramatically reduce the point cloud size in $O(n \log n)$ time (\cref{alg:filter}).

However, nearby points in a higher-dimensional space are not guaranteed to be adjacent in a fixed space-filling curve.
We mitigate this by checking for neighbors on multiple space-filling curves, one for each permutation of dimensions.
If two points are near to each other, then it is likely that at least one such curve will place them adjacent to one another.
Repeating the filtering process on each permutation of dimensions means that the filtering procedure scales with $O(k! n \log n)$, but for $k=3$, there are only six such permutations, so the cost of extra filter checks is minimal compared to the savings in point cloud size.

\begin{algorithm}
    \caption{Filter}
    \label{alg:filter}
    \KwIn{list of points $\mathit{PC}$, filter radius $r_\text{filter}$}
    \KwOut{list of filtered points $\mathit{PC}' \subseteq \mathit{PC}$}

    $\mathit{PC}' \leftarrow \mathit{PC}$;

    \ForEach{permutation $X$ of dimensions}{
    
        sort \add{$\mathit{PC}'$} along a Z-order curve by dimension order $X$;
        
        $i \leftarrow 0$;
        
        \ForEach{$j \leftarrow 1, 2, \cdots |PC'| - 1$}{
            \If{$\|\mathit{PC}'_i - \mathit{PC}'_j\| > r_\text{filter}$}{
                $i \leftarrow i + 1$;
                
                $\mathit{PC}'_i \leftarrow \mathit{PC}'_j$;
            }
        }

        $\mathit{PC}' \leftarrow \mathit{PC}'_{0 \cdots i}$;
    }

    \Return $\mathit{PC}'$;
\end{algorithm}

Additionally, we filter out any points which we can prove will never be in collision with the robot. 
This process is simple for fixed-base arm robots: a point can only collide with a robot if its distance to the base link of the robot is less than the maximum extension length of the arm.

All together, this filter can achieve dramatic reductions in point cloud size even for small values of $r_\text{filter}$, filtering clouds with over a hundred thousand points to less than ten thousand in a few milliseconds.
As shown in \cref{fig:pc}, relatively conservative values of $r_\text{filter}$ cull the point cloud by a dramatic amount; additionally, the filtering process significantly reduces point cloud density, reducing the expected colliding-set size and therefore improving tree construction and query times.

\subsection{Collision querying}\label{sec:search}

When collision-checking for a robot, a robot's configuration is valid only if all of the robot's physical geometry is not in collision with the environment. 
If any sphere of the robot's geometry is in collision with the environment, then the entire configuration is invalid.
This provides us with an early-termination condition: we need only find a single colliding sphere to invalidate an entire configuration.
By parallelizing collision checks across multiple query spheres, the entire search can terminate as soon as one collision is found.

The first step of a collision check is a search through the tree, as outlined in lines 1-7 of \cref{alg:search}.
Given some query sphere with center $x$ and radius $r$, the search begins with a test index $i=0$ and dimension $d=0$. 
Then, at each step of the search, if $x_d < T_i$, $i$ is updated to $2i + 1$, or $2i + 2$ otherwise, while $d$ is updated to $d + 1 \mod k$. 
This is the same update rule for the Eytzinger layout as used during construction: index $2i + 1$ corresponds to the left sub-tree, while index $2i + 2$ corresponds to the right subtree.
When the search completes, the final value of $i$ is an integer in the range $[n - 1, 2n - 1)$, with each value corresponding to a unique leaf of the tree. $i - n + 1$ is an integer in the range $[0, n)$ corresponding to each point stored at a leaf of the tree. 
This traversal can be performed branchlessly by converting the boolean value $x_d \geq T_i$ comparison into an integer $l$, assigning $i \leftarrow 2i + 1 + l$. 
Since $n$ is a power of two, all traversals of the tree terminate in exactly the same number of iterations, and we do not need to test for reaching the end of the tree.
Likewise, the memory access pattern of the traversal is extremely predictable: all accesses at a given iteration are restricted to a small set of possible values.
The branchless traversal sequence allows for efficient \simd{} parallelism: each lane of a single register contains a different test index, and each load, comparison, and index update is parallel, providing a large improvement in performance.

After determining which cell contains $x$, the search algorithm first performs an efficient collision check between the query sphere and the axis-aligned bounding box $A_{i - n + 1}$, which contains all points afforded by the cell. 
This step occurs in lines 8-9 of \cref{alg:search}.
This check does not change the final output of the search algorithm; instead, it simply filters out spheres which can be trivially proven not to collide with any points to reduce the number of expensive traversals through the affordance set.
This step is once again parallelizable across multiple queries: for each \simd{} lane of the processor, we can compute the distance from each query sphere to its leaf's bounding box in parallel instead of sequentially.
During a parallel query, if spheres are not in collision with the bounding box, then they can be masked out from all remaining collision checks, reducing the overall search time.
If no spheres in the query set collide with the axis aligned bounding box, then the query set is provably not in collision, and the search can terminate immediately.

Finally, the collision-checking procedure exhaustively checks for collision between the query sphere and all points $p \in P_{i - n + 1}$. If $x$ is nearer to any $p$ than $r$, then it is in collision.
This step of collision-checking occurs in lines 10-13 of \cref{alg:search}.
We parallelize this step differently from the previous two: instead of parallelizing across the set of query spheres, we parallelize across the set of test points in each affordance set.
This is all for cache locality: parallelizing across the test points means that all memory accesses are in the same contiguous region, instead of requiring inefficient gather instructions across multiple different affordance sets.
Such parallelism would not be possible with a conventional $k$-d tree, as the set of possibly-colliding points is not known to the search until it explores each sub-tree.

\begin{algorithm}
    \caption{CollisionCheck}
    \label{alg:search}
    \KwIn{\capt{} $(T, A, P)$ containing $n$ points in $\mathbb{R}^k$, sphere $s$ with center $x$ and radius $r$}
    \KwOut{Whether $s$ collides with any point in the tree}
    
    $d, i \leftarrow 0$;
    
    \While{$i < n - 1$}{
    
        \If{$x_d \leq T_i$}{
        
            $i \leftarrow 2i + 1$;
            
        }
        \Else{
        
            $i \leftarrow 2i + 2$;
            
        }
        
        $d \leftarrow d + 1 \mod k$;
    }

    \If{$A_{i - n + 1}$ does not intersect $s$}{
        \Return false;
    }

    \For{$p \in P_{i - n + 1}$}{
        \If{$\|x - p\| \leq r$}{
            \Return true;
        }
    }

    \Return false;
    
\end{algorithm}

\section{Analysis}\label{sec:analysis}
\subsection{Runtime}

The runtime of the construction procedure is dictated by two sub-procedures: the partitioning of the space, which is the same as a $k$-d tree at $O(k n \log n)$ for a point cloud with $n$ points; and the construction of the final affordance set, which requires $O(ka)$ time for each point, where $a$ is the maximum size of an affordance set. 
Therefore the total runtime of construction is $O(k n \log n + k n a)$. 
When the dispersion of the point cloud is high, $a$ tends to be small, so the construction runtime is $O(k n \log n)$. 
However, for extremely low-dispersion point clouds, all points in the point cloud are afforded by each cell, so $a = O(n)$, yielding a much larger construction runtime of $O(k n^2)$. 
In total, a \capt{} consumes $O(kna)$ memory, which may be as much as $O(k n^2)$ for extremely low-dispersion clouds.

In total, each collision query against the \nameoftechnique{} performs $O(\log n)$ comparisons while traversing the tree, then $O(a)$ checks evaluating the distance to each point $p \in P_i$, where $a$ is the maximum size of any affordance set $P_i$. Therefore the total search procedure requires $O(\log n + k a)$ steps.

\subsection{Correctness}

\begin{lemma}
    \label{lem:collision}
    If $P_i$ is the set of points corresponding to a cell $c$ of the tree, then any sphere $s$ with center $x$ contained by $c$ and with radius $r \in [r_\text{min}, r_\text{max}]$ collides with a point in the point cloud $\mathit{PC}$ if and only if $s$ collides with a point in $P_i$.
\end{lemma}
\begin{proof}
    Since $P_i \subseteq \mathit{PC}$, it is trivial to show that if $s$ collides with a point in $P_i$, then it collides with a point in $\mathit{PC}$.
    Now we must prove the converse\add{;} that is,  $\exists p \in \mathit{PC} : \|x - p\| \leq r \rightarrow \exists q \in P_i : \|x - q\| \leq r$. Let $t$ be the single point in $\mathit{PC}$ contained by $c$.
    
    Case 1: $\exists q \in c : \|q - t\| > r_\text{min}$. Then, by construction, $P_i = \{ p \in \mathit{PC} : \min_{x_c \in c} \|x_c - p \| \leq r_\text{max}\}$. If $\exists p \in \mathit{PC} : \|x - p \| \leq r$, then $\min_{x_c \in c} \|x_c - p\| \leq r \leq r_\text{max}$. Therefore if the sphere centered at $x$ collides with $p$, then $x$ collides with a point in $P_i$.

    Case 2: $\nexists q \in c : \|q - t\| > r_\text{min}$. By construction, $P_i = \{t\}$. Since $x \in c$, $\|x - t \| \leq r_\text{min} \leq r$, so if $s$ is in collision with $\mathit{PC}$, then $s$ is in collision with $P_i$.
\end{proof}

\begin{lemma}
    \label{lem:filter}
    There exists a filter radius $r_\text{filter}$ such that the filtering process does not insert a gap in the point cloud $\mathit{PC}$ larger than the minimum collision-check sphere diameter $2 r_\text{min}$.
\end{lemma}
\begin{proof}

    \add{
        Let $O \subseteq \mathbb{R}^k$ be the ground truth obstacle volume for a robot environment, and let  $\mathit{PC} \subseteq O$ be its surface approximation as a point cloud. 
        Then the $L_2$ dispersion $\delta(O, \mathit{PC})$, defined following~\citet{lavalle2006planning}, is:
    }
    \begin{equation*}
        \add{\delta(O, \mathit{PC}) = \sup_{x \in O} \min_{p \in PC} \|x - p\|}
    \end{equation*}
    \add{
        That is, any point $x$ contained in $O$ must be no further than $\delta(O, \mathit{PC})$ from some point $p$ contained in $\mathit{PC}$.
        }

    \add{
        Now consider some sphere $s$ with diameter $2r_\text{min}$. 
        To fit in a gap of width $2r_\text{min}$ in $PC$, $s$ must have center $x \in O$.
        Then, by the definition of dispersion, $\exists p \in \mathit{PC} : \|x - p \| \leq \delta(O, \mathit{PC})$.}

    \add{
        The filter algorithm presented in~\cref{alg:filter} only removes a point $p \in \mathit{PC}$ if $\exists p^* \in \mathit{PC}': \|p - p^*\| \leq r_\text{filter}$; that is, if there is already a retained point closer than $r_\text{filter}$ to the point in question.
        Thus, the largest gap that the filter can introduce is bounded by $2\left(r_\text{filter} + \delta(O, \mathit{PC})\right)$.
        So, it follows that $\exists p^* \in \mathit{PC}' : \|x - p^*\| \leq r_\text{filter} + \delta(O, \mathit{PC})$.
        To prevent introducing gaps larger than $2r_\text{min}$, we then require that:
        \begin{equation*}
            \|x - p^*\| \leq r_\text{filter} + \delta(O, \mathit{PC}) \leq r_\text{min}
        \end{equation*}
        Rearranging, we reach $r_\text{filter} \leq r_\text{min} - \delta(O, \mathit{PC})$.
        }
\end{proof}

\add{The choice of $\mathit{r_\text{filter}}$ as outlined by \cref{lem:filter} ensures that the center of any sphere of the robot does not intersect the obstacle $O$, preventing the robot from travelling directly through an obstacle's surface.
However, like any filtering scheme, this may somewhat reduce the envelope of the point cloud.
The simplest way to overcome this issue is to instead substitute all query radii $r_q$ with $r_q' = r_q + r_\text{filter}$. 
An alternate interpretation of this is that all points in the filtered point cloud are padded by $r_\text{filter}$ to become solid, volumetric spheres. 
With this padding, any value of $r_\text{filter}$ may be used.}

\add{In practice, however, we find empirically that aggressively padding queries or selecting a small $r_\text{filter}$ is overly conservative (\ie{} unnecessary for maintaining plan validity) and detrimental to \capt{} construction performance, so more aggressive filtering is possible.}

The lemmas above ensure that collision-checking using a \nameoftechnique{} does not change problem feasibility compared to brute-force collision-checking; \ie{} a planner using a \nameoftechnique{} will not erroneously report that a problem is solvable or unsolvable due to the collision-checking backend.
Lemma \ref{lem:collision} implies that a \nameoftechnique{} does not alter the collision status of any query sphere, so any trajectory which is valid through brute force collision checking is also valid when using a \nameoftechnique{}.
Likewise, Lemma \ref{lem:filter} implies that the filtering procedure is sufficiently conservative to avoid creating spurious gaps, meaning that the filtering process does not make invalid plans feasible, subject to the correct radius padding.

\section{Experiments}\label{sec:experiments}
\begin{figure*}
    \centering
    \includegraphics[width=\textwidth]{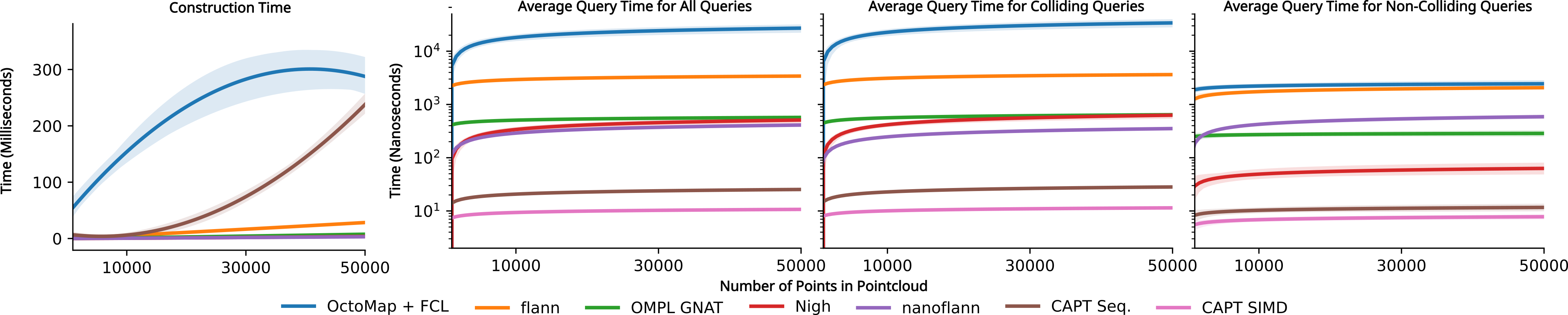}
    \caption{Construction times and average query throughput for pointcloud collision-checking approaches.
    Methods were evaluated on exemplary pointclouds from each of the 7 benchmark datasets from MotionBenchMaker~\cite{chamzas2021mbm}, and evaluated against the set of all collision queries attempted by the motion planner.
    On the left, construction times are presented in milliseconds.
    On the right, average query time in nanoseconds are shown on a \textbf{logarithmic} scale for all queries, only colliding queries, and queries that are not in collision.
    99\% confidence intervals are shown over 2\textsuperscript{nd}-order polynomial fitted curves for build time and linear log fit curves for query time.
    }
    \label{fig:throughput}
\end{figure*}

\begin{table*}
  \centering
  \fontsize{6pt}{6pt}\selectfont
  \begin{tabular}{r | r | r | r r r | r r r | r | r r r | r}
  & Backend & Mean Filter & Mean Build & Med. Build & 95\% Build & Mean Plan & Med. Plan & 95\% Plan & Mean Simpl. & Mean Total & Med. Total & 95\% Total & Succ. \\
  \hline
\hline
\parbox[t]{0.1mm}{\multirow{5}{*}{\rotatebox[origin=c]{90}{UR5}}}
& OctoMap         &      \multirow{4}{*}{2.897} &    120.935 &    104.892 &    218.075 &     78.894 &     24.736 &    375.782 &     23.290 &    220.906 &    179.170 &    521.604 &  96.5\%\\
& nanoflann       &            &      0.346 &      0.339 &      0.596 &     15.526 &      4.797 &     78.979 &      7.810 &     26.730 &     15.806 &     92.221 & 100.0\%\\
\arrayrulecolor{gray}\cline{4-6}\arrayrulecolor{black}
& \capt{} Seq.    &            &      \multirow{2}{*}{5.859} &      \multirow{2}{*}{5.608} &     \multirow{2}{*}{10.467} &      0.833 &      0.273 &      4.226 &      0.368 &      9.935 &      8.829 &     16.279 & 100.0\%\\
& \capt{} \simd{}    &            &           &           &           &      \textbf{0.490} &      \textbf{0.146} &      \textbf{2.542} &      0.225 &      \textbf{9.499} &      \textbf{8.630} &     \textbf{15.541} & 100.0\%\\
\arrayrulecolor{gray}\cline{2-14}\arrayrulecolor{black}
& Primitives      &          - &          - &          - &          - &      0.204 &      0.051 &      1.039 &      0.067 &      0.272 &      0.117 &      1.145 & 100.0\%\\
\hline
\parbox[t]{0.1mm}{\multirow{5}{*}{\rotatebox[origin=c]{90}{Panda}}}
& OctoMap         & \multirow{4}{*}{2.712} &     66.127 &     49.503 &    122.043 &     25.579 &      9.363 &    109.210 &     24.424 &    118.844 &    103.436 &    223.829 &  99.6\%\\
& nanoflann       &            &      0.315 &      0.233 &      0.611 &      5.958 &      3.241 &     20.352 &      7.423 &     16.607 &     13.726 &     39.647 & 100.0\%\\
\arrayrulecolor{gray}\cline{4-6}\arrayrulecolor{black}
& \capt{} Seq.    &            &      \multirow{2}{*}{4.269} &      \multirow{2}{*}{3.227} &      \multirow{2}{*}{8.758} &      0.342 &      0.182 &      1.166 &      0.477 &      7.781 &      6.430 &     15.078 & 100.0\%\\
& \capt{} \simd{}    &            &           &           &           &      \textbf{0.198} &      \textbf{0.102} &      \textbf{0.728} &      0.261 &      \textbf{7.445} &      \textbf{6.055} &     \textbf{14.704} & 100.0\%\\
\arrayrulecolor{gray}\cline{2-14}\arrayrulecolor{black}
& Primitives      &          - &          - &          - &          - &      0.078 &      0.034 &      0.363 &      0.070 &      0.148 &      0.100 &      0.468 & 100.0\%\\
\hline
\parbox[t]{0.1mm}{\multirow{5}{*}{\rotatebox[origin=c]{90}{Fetch}}}
& OctoMap         &      \multirow{4}{*}{3.345} &    160.976 &    124.376 &    335.959 &    309.302 &    212.953 &    894.198 &    121.892 &    429.362 &    375.223 &    900.221 &  96.6\%\\
& nanoflann       &            &      0.469 &      0.391 &      0.863 &    177.912 &     70.510 &    661.252 &     33.131 &    204.500 &    102.742 &    680.011 &  99.9\%\\
\arrayrulecolor{gray}\cline{4-6}\arrayrulecolor{black}
& \capt{} Seq.    &            &      \multirow{2}{*}{6.073} &      \multirow{2}{*}{5.145} &     \multirow{2}{*}{10.508} &     18.331 &      3.863 &     76.251 &      1.831 &     29.624 &     16.015 &     88.011 &  99.7\%\\
& \capt{} \simd{}    &            &           &           &           &     \textbf{10.736} &      \textbf{2.159} &     \textbf{42.974} &      0.983 &     \textbf{21.213} &     \textbf{12.648} &     \textbf{58.239} &  99.7\%\\
\arrayrulecolor{gray}\cline{2-14}\arrayrulecolor{black}
& Primitives      &          - &          - &          - &          - &      3.873 &      0.779 &     16.290 &      0.262 &      4.136 &      0.966 &     16.744 &  99.3\%\\
  \end{tabular}
  \caption{Statistics over the MotionBenchMaker~\cite{chamzas2021mbm} dataset.
   We compare parallel (\capt{} \simd{}) and sequential (\capt{} Seq.) \nameoftechnique{} collision checking against other collision checking backends. 
  All collision-checking backends except for the primitive geometry used the same filtered point clouds for planning.
  We report the mean, median, and 95th percentile times spent constructing each collision-checking data structure, planning, simplifying the path, and the total time spent from observation to completed plan for each robot and collision-checking backend.
  All times are in \textbf{milliseconds}.}
  \label{table:mbm}
  \end{table*}

We benchmarked collision-checking throughput on an \amd Ryzen\texttrademark~9 7950X \cpu clocked at 4.5GHz against six different collision-checking and nearest-neighbor implementations. 
We compared against OctoMaps~\cite{Hornung2013}, a voxel-based method, backed by \fcl{}~\cite{Pan2012a}; Nigh~\cite{ichnowski2018nigh} and \nanoflann{}~\cite{blanco2014nanoflann}, $k$-d tree implementations; \gnat{}~\cite{brin1995gnat}, a hyperplane partitioning tree (as implemented in the Open Motion Planning Library~\cite{Sucan2012}); and \flann{}~\cite{muja2009flann}, an approximate nearest-neighbor library.
Our approach was implemented in C++ and integrated with an existing vector-accelerated motion planning framework \cite{Thomason2023VAMP}. 
This planner represents the robot as a hierarchy of spheres, so all robot-environment collision-checking was performed using the \capt{}. 
We also implemented a collision-checking backend using sequential queries against the tree.
Here, sequential refers to collision checking each of the $n$ spheres packed into a \simd{} vector sequentially, rather than using \simd{} intrinsics to evaluate, thus demonstrating the benefits of \simd{} parallelism.
Sequential queries are used as well for \nanoflann{} and OctoMap collision checking.

We benchmark planning performance on the challenging MotionBenchMaker~\cite{chamzas2021mbm} dataset, with 3 robots (the 6-\dof UR5, the 7-\dof Panda, and the 8-\dof Fetch) in 7 scenes (\emph{table pick}, \emph{table under pick}, \emph{box}, \emph{cage}, \emph{bookshelf small}, \emph{bookshelf tall}, and \emph{bookshelf thin}) each, performing 100 different planning problems per scene.

All motion planning was performed on a single thread, using an implementation of a dynamic-domain~\cite{Jaillet2005} balanced~\cite{kuffner2005balanced} RRT-Connect~\cite{Kuffner2000} limited to 1 million iterations.
All plans used the same sequence of randomly sampled configurations, so any difference in performance is from collision-checking speed, not from sampling order.
All code was compiled with \verb|clang 16.0.6| using the \verb|-O3| compiler optimization level along with native \cpu{} optimizations.

\subsubsection{\add{Collision query throughput}}
\add{
    We began by evaluating collision-checking throughput on all of the possible backends.
    First, we recorded the set of all collision-checking queries made by a motion planner using ground-truth primitive geometry on each scene from the MotionBenchMaker~\cite{chamzas2021mbm} dataset.
    We then generated point clouds for each scene by uniform random sampling of the geometry's surface, then filtered each point cloud with $r_\text{filter} \in [1 \si{\milli \meter}, 10 \si{\centi \meter}]$ to reach a desired size.
    All collision-checking throughput experiments were performed on the same set of point clouds.
    Finally, we executed the exact same queries on each collision-checking method, recording the total timing for collision-checking and avoiding any other timing overhead from other steps in the motion planning process.
    \capts{} were constructed with $r_\text{min} = 1\si{\centi \meter} $ and $r_\text{max} = 8\si{\centi \meter}$.
    OctoMaps were constructed with a resolution of 1\si{\centi \meter}.
    FLANN indices were created with 4 $k$-d trees.
    All experiments were performed on a single \cpu{} thread.
}

\subsubsection{\add{Motion planning performance}}
\add{
    We implemented full motion-planning backends using OctoMaps and NanoFLANN (as it had the the next-highest throughput, after \capts{}, of any collision-checking method).
    Each backend was used as part of the same motion planning system, so all speedups are due to collision-checking speed, not sampling order or other aspects of planner efficiency.
    We compared the relative performance of $\capts{}$, using both sequential and parallelized \simd{} queries, with these two backends.
    All point clouds were filtered with $r_\text{filter} = 2\si{\centi \meter}$; although this filter radius is greater than that suggested by \cref{lem:filter}, it was empirically tested to strike an appropriate balance between performance and fidelity, \ie{} not allowing invalid plans.
    See~\cref{app:rfilter} for further experiments on the effect of $r_\text{filter}$.
    \capts{} were constructed with $r_\text{min}$ and $r_\text{max}$ derived from robot geometry; $(r_\text{min}, r_\text{max})$ was equal to $(1.5\si{\centi \meter}, 8\si{\centi \meter})$, $(1.2\si{\centi \meter}, 6\si{\centi \meter})$, and $(1.2\si{\centi \meter}, 5.5\si{\centi \meter})$ for the UR5, Panda, and Fetch respectively.
    OctoMaps were constructed with a resolution of 1\si{\centi \meter}.
    Once again, these tests were performed exclusively on a single \cpu{} thread to isolate per-thread performance; we leave thread-level parallelization to future work.
}

\subsection{Empirical Results}\label{sec:results}

\subsubsection{\add{Collision query throughput}}
\label{subsec:throughputresults}
\mf{fig:throughput} presents timing results for \add{\capt{}} construction times and average query throughput times for three different classes of tests: all-colliding queries are a set of queries where each sphere in the query set collides; non-colliding queries have no sphere in collision, and mixed queries are a mix of all-colliding, non-colliding, and partially-colliding queries.
Queries against the point cloud were created by recording the set of all queries made by a motion planner in the scene, then recording the runtime of \add{checking the same sequence of queries} \add{against} each collision-checking system.
We observe that \capt{} construction is significantly slower than other tree-based nearest neighbor data structures, but is still faster than an OctoMap for construction on point cloud data. 
The construction procedure exhibits significantly superlinear scaling, showing that point cloud filtering is necessary to use a \capt{} effectively.

Collision queries against the tree are overwhelmingly faster than any other data structure, running nearly ten times faster than the nearest data structures for collision checking, for an average performance of 9.89 nanoseconds per query---in comparison, the next best performing approach, \nanoflann{}, takes on average 309 nanoseconds per query.
Remarkably, collision checks against an OctoMap are over three orders of magnitude slower than against a \capt{} (averaging 0.01 milliseconds per query).
This demonstrates a need for the field to reevaluate methods for planning with sensor data: it is possible to achieve extremely high performance gains with a different data structure.

\subsubsection{\add{Motion planning performance}}\label{subsec:simexp}
\Cref{table:mbm} shows some critical statistics for filtering, collision checking data structure construction, planning, simplification, and total time for the 6-\dof UR5, 7-\dof Panda, and 8-\dof Fetch over the MotionBenchMaker dataset, as described above.
These benchmarks show a dramatic improvement in performance over baseline approaches, with \simd{} collision checking with a \capt{} demonstrating planning times on par with the ground-truth primitive-based planner.
For both the UR5 and Panda arms, the 95\% quantile of total time end-to-end (filtering, building the \capt{}, planning, and simplification) takes less than 16 milliseconds, faster than a 60FPS camera can refresh and provide a new pointcloud.
We also highlight that \capts{} provide such an enormous speedup that \textbf{motion generation is no longer the most expensive step}. Instead, other steps in the planning pipeline dominate planning times: point cloud filtering and \capt{} construction account for the lion's share of planning time.

\subsection{Planning from real sensor data}\label{subsec:realrobot}

Finally, we applied our planning system to \add{point clouds observed} from a real-world scene with an Intel RealSense D455 RGB-D camera\add{. }\cref{fig:realsense} \add{shows an example of these data}.
\subsubsection{\add{Static scene}}
We \add{first} created a planning problem in \add{a static snapshot of} this scene, requiring a UR5 robot to move from its initial pose to a "reach" point across the table.
The original point cloud contained 166587 points; applying our space-filling curve filter (\cref{sec:filtering}) with $r_\text{filter} = 2 \si{\centi \meter}$ reduced the cloud down to 2732 points.
\add{
    In this experiment, we used a minimum radius $r_\text{min} = 1.5 \si{\centi \meter}$ to match the geometry of our model of the UR5. 
    As in~\cref{subsec:simexp}, we chose to use a larger $r_\text{filter}$ than suggested by \cref{lem:filter} because it empirically did not reduce plan quality; for timings with different values of $r_\text{filter}$, see \cref{app:rfilter}.} 
Running on the same machine as our other experiments, we observe a median planning time of 215 microseconds, with a simplified path returned in a total of 575 microseconds.
The total duration from the start of point cloud filtering through \capt{} construction, planning, and simplification was a median of 7.166 milliseconds---corresponding to a complete planning rate of roughly 140Hz.

\subsubsection{\add{Dynamic scene}}
\add{We additionally evaluated our planning system in a live control loop on the UR5, tasking it with moving between a sequence of preset goal waypoints while dodging unmodeled dynamic obstacles (\ie{} pool noodles moved by humans to obstruct the robot).
Planned trajectories are passed to a simple velocity interpolation controller for time parameterization and execution, and are replaced and updated on every new point cloud.
We observe that the system is able to plan at or above the 60FPS camera frame rate; qualitatively, this speed enables the robot to reactively dodge obstacles and effectively maneuver in the scene, despite a lack of motion forecasting or obstacle modeling.
We report additional statistics on planning performance and point cloud properties for this experiment in~\cref{app:realrobot}; please also see the supplementary material for a video showing the robot in action.}

\section{Conclusion}\label{sec:conclusion}

Planning from sensor data is a crucial component of autonomous robotics.
In this paper, we present a novel data structure for motion planning with observed point clouds, demonstrating an order-of-magnitude speedup compared to state-of-the-art techniques.
We also present a unique filtering algorithm to reduce the density of a point cloud while still providing safety guarantees on collision detection.
Combined, these two contributions enable a robot to plan from sensor data in milliseconds on a single CPU core, allowing the robot to plan faster than standard 60FPS camera refresh rates.
This means that \emph{robots can now use sampling-based motion planning in real time on purely sensed environments, using only general-purpose low-power hardware.}

The primary limitation of a \capt{} is that it is an immutable data structure. 
After construction, no points in the tree can be inserted or deleted.
\add{Since depth-camera images are streamed on a frame-by-frame basis, this is not a problem for collision-checking in dynamic environments, since we can reconstruct the \capt{} from scratch for each frame.}
\add{However, the \capt{}'s immutability} precludes use of a \capt{} for the state-space nearest-neighbor search required by most sampling-based planning algorithms.
Future extensions to the \capt{} structure could enable incremental updating, which would allow it to be used for nearest-neighbor search in the state space during sampling-based planning.
\add{Unlike OctoMaps~\cite{Hornung2013}, the \capt{} does not distinguish between free and unobserved space.
This may be problematic for cluttered environments due to occlusions, and in future work we are interested in extending the \capt{} to more directly model visibility and occlusion, as well as to better handle perceptual uncertainty by \eg{} modeling points as probabilistic particles.}

As well, although our choice to duplicate potentially colliding points to construct the affordance sets enables \capts{} to avoid branches and effectively exploit parallelism, it also limits their capacity for scaling to massive point clouds, such as those constructed by autonomous vehicles.
In future work, we would be interested in exploring techniques for compressing or otherwise de-duplicating affordance sets.
Perhaps more promising is the potential for using \capts{} as a secondary collision data structure paired with another form of spatial subdivision, such as a spatial hash or voxel grid~\cite{Oleynikova2017,Millane2023}.
This hierarchical fused data structure would allow a set of \capts{} to each be ``responsible'' for only a local neighborhood of a large point cloud while maintaining efficient and parallelizable queries over the entire cloud.

Finally, our performance results challenge conventional assumptions about the nature of planning.
We have demonstrated that judicious application of parallelism and insights into the core problems of collision checking against sensor data enables extraordinary improvements in planning time, so much so that motion planning from sensor data could now be seen as a cheap primitive operation, instead of a time-consuming bottleneck.

\printbibliography{}

\clearpage

\appendix
\subsection{\add{Empirical impact of filter radii}}

\label{app:rfilter}
\add{We ran the simulated planning experiments from~\cref{subsec:simexp} for the Panda robot with different values for the $r_\text{filter}$ parameter to~\cref{alg:filter} to investigate its effect on \capt{} construction times, planning performance, and problem feasibility.
These results are shown in~\cref{table:mbm2}.
We note that, although the \capt{} is sensitive to the value of $r_\text{filter}$ in its construction time in particular, even for conservative values of $r_\text{filter}$ (\eg{} matching or less than the bound suggested by~\cref{lem:filter}), our planning times are always faster than any baseline, and our total times are competitive or fastest.
Note also that the baselines are only evaluated with the most aggressive value of $r_\text{filter}$, and would also be slowed in planning and total time for more conservative values.
Finally,~\cref{table:mbm2} also validates that, for all values of $r_\text{filter}$, the paths we find are valid \textbf{with respect to the exact, primitive geometric obstacles}, evaluated post hoc.}

\subsection{\add{Empirical dispersion of point clouds}}

\add{
    We recorded the values of dispersion $\delta(O, \mathit{PC})$ as discussed in \cref{lem:filter} across the point clouds used for our simulation collision query throughput and planning performance experiments.
    We compute dispersion by using a standard nearest-neighbors data structure to query the distance to the closest neighbor of each point in each point cloud for each scene, and recording basic summary statistics, as shown in~\cref{table:dispersion}.
    We note that $\delta(O, \mathit{PC})$ is often quite small for well-observed obstacles; further, as argued in the proof sketch for~\cref{lem:filter}, the increase in dispersion resulting from applying~\cref{alg:filter} is bounded by $r_\text{filter} + \delta(O, \mathit{PC})$.
    Therefore, selecting an $r_\text{filter}$ in the same order of magnitude as $r_\text{min}$ is a reasonable choice.
}

\subsection{\add{Real-robot planning experiments}}\label{app:realrobot}
\label{app:dispersion}

\add{
We evaluate the impact of the value of $r_\text{filter}$ via the real-robot demonstration of planning with \capts{} discussed in \cref{subsec:realrobot}.
We collected 300 observed point clouds from sequential frames of RGB-D video generated by an Intel Realsense D455 sensor, and computed the mean post-filter point cloud sizes and essential timing statistics (\ie{} timing for applying the filter, building a \capt{} on the filtered point cloud, solving a motion planning problem with the \capt, and simplifying the solution found---which requires further collision-checking) for a range of values of $r_\text{filter}$, keeping the values of $r_\text{min}$ and $r_\text{max}$ constant at 1.5\si{\centi \meter} and 8\si{\centi \meter} respectively.
\Cref{table:rfilter} shows these quantitative results, demonstrating that although \capts{} remain fast at conservatively small values of $r_\text{filter}$, increasing $r_\text{filter}$ dramatically decreases both point cloud size and \capt{} construction time.
}

\add{
Crucially, we also qualitatively find that, for any $r_\text{filter} \leq 2\si{\centi \meter}$, the generated trajectories are valid and do not intersect or contact any obstacles.
As such, broadly speaking, a user can vary the value of $r_\text{filter}$ to trade fidelity of representation for performance, and reasonable balances of the two are easy to find.
}

\begin{table}
  \centering
  \fontsize{8pt}{8pt}\selectfont
  {\begin{tabular}{ r | r r r | r r r }
  & \multicolumn{3}{c|}{$\delta(O, \mathit{PC})$} &  \multicolumn{3}{c}{$\delta(O, \mathit{PC}')$} \\
    $r_\text{filter}$ & Mean & Median & 95\% & Mean & Median & 95\% \\
    \hline
0.5 & \multirow{6}{*}{0.00336} & \multirow{6}{*}{0.00198} & \multirow{6}{*}{0.01075} & 0.00763 & 0.00675 & 0.01314 \\
1 &  &&  & 0.01253 & 0.01182 & 0.01774 \\
1.5 &  &&  & 0.01758 & 0.01696 & 0.02314 \\
1.8 & &&  & 0.02055 & 0.01995 & 0.02663 \\
1.9 &  &&  & 0.02158 & 0.02097 & 0.02790 \\
2 & && & 0.02261 & 0.02199 & 0.02919 \\
  \end{tabular}}
    \caption{\add{Empirical measurements of point cloud dispersion before (left grouping) and after (right grouping) applying the filter proposed in~\cref{alg:filter} with a range of $r_\text{filter}$ values (leftmost column, in \si{\centi\meter}).
    Pre-filter values are identical and accumulated over all MotionBenchMaker problems used for evaluation with the Panda robot; all dispersion values are given in \si{\centi\meter}.}}
    \label{table:dispersion}
\end{table}

\begin{table}
\centering
{\begin{tabular}{r|c|rrrr|r}
$r_\text{filter}$ & $|PC'|$ & Filter & Build & Plan & Simpl. & Total \\ \hline
1&16225 & 6.944 & 30.373 & 0.168 & 0.275 & 37.761 \\
1.5& 7872 & 5.682 & 7.669 & 0.079 & 0.085 & 13.517\\
2&  4614 & 4.990 & 3.964 & 0.087 & 0.089 & 9.131 \\
\end{tabular}}
\caption{\add{Impact of $r_\text{filter}$ (in centimeters) on planning times for the real robot experiment with the UR5. The initial pointcloud size $|PC|$ is always 307200 (constructed from a 640x480 pixel depth image).
We report mean point cloud sizes, mean point cloud filtering times, mean \capt{} construction times, mean motion planning times, and mean path simplification times, as well as the mean total end-to-end planning time. 
$r_\text{filter}$ is in \textbf{centimeters}, and all times are in \textbf{milliseconds}.}}
\label{table:rfilter}
\end{table}

\onecolumn
\begin{table*}[h]
  \centering
  \fontsize{6pt}{6pt}\selectfont
  {\begin{tabular}{ r | r | r r r | r r r | r | r r r | r}
  $r_\text{filter}$ & Mean Filter & Mean Build & Med. Build & 95\% Build & Mean Plan & Med. Plan & 95\% Plan & Mean Simpl. & Mean Total & Med. Total & 95\% Total & Succ. \\
\hline
0.5 &      4.544 &     62.560 &     56.294 &    122.966 &      0.427 &      0.210 &      1.581 &      0.480 &     68.012 &     61.546 &    132.708 & 100.0\\
1 &      3.307 &     19.339 &     13.286 &     41.888 &      0.288 &      0.136 &      1.136 &      0.344 &     23.280 &     16.514 &     49.667 & 100.0\\
1.5 &      2.742 &      8.069 &      5.954 &     16.595 &      0.228 &      0.112 &      0.923 &      0.288 &     11.329 &      8.753 &     22.851 & 100.0\\
1.8 &      2.514 &      5.378 &      4.278 &     10.895 &      0.212 &      0.103 &      0.923 &      0.265 &      8.370 &      6.905 &     16.565 & 100.0\\
1.9 &      2.454 &      4.701 &      3.613 &      9.760 &      0.207 &      0.102 &      0.734 &      0.265 &      7.628 &      6.281 &     15.149 & 100.0\\
2 &      2.400 &      4.098 &      3.125 &      8.467 &      0.192 &      0.099 &      0.702 &      0.253 &      6.945 &      5.658 &     13.812 & 100.0\\
  \end{tabular}}
    \caption{\add{Effect of $r_\text{filter}$ on MotionBenchMaker planning performance for the Panda robot for \capt{} \simd{}.
    All planning results are valid with respect to the underlying primitive scene representation sampled to generate simulated pointclouds.
    $r_\text{filter}$ is given in \textbf{centimeters}, and all times are given in \textbf{milliseconds}.
    }}
  \label{table:mbm2}
\end{table*}

\end{document}